\documentclass[letterpaper]{article} 
\usepackage{aaai18}  
\usepackage{times}  
\usepackage{helvet}  
\usepackage{courier}  
\usepackage{url}  
\usepackage{graphicx}  
\usepackage{multirow}
\usepackage{latexsym}
\usepackage{tabularx}
\usepackage{array}
\usepackage{enumitem}
\usepackage{url}
\usepackage{amsmath, amsfonts, amssymb}
\usepackage{mathtools}
\usepackage{multirow}
\usepackage{algorithm}
\usepackage[noend]{algpseudocode}
\usepackage{capt-of}
\usepackage{array, booktabs}
\usepackage{mydefs}
\usepackage{macros}

\usepackage{color}

\usepackage{footmisc}
\DefineFNsymbols{mySymbols}{{\ensuremath\dagger}{\ensuremath\ddagger}\S\P
	*{**}{\ensuremath{\dagger\dagger}}{\ensuremath{\ddagger\ddagger}}}
\setfnsymbol{mySymbols}

\begin{document}
\title{Scalable Structure Learning for Probabilistic Soft Logic}
\author{Varun Embar\thanks{These authors contributed equally.} \and Dhanya Sridhar$^\dagger$
	\and Golnoosh Farnadi
	\and Lise Getoor
\\ University of California Santa Cruz \\ \texttt{\{vembar,dsridhar,gfarnadi,getoor\}@ucsc.edu}}

%

%
%
%
%
%

\maketitle              
\begin{abstract}

Statistical relational frameworks such as Markov logic networks and probabilistic soft logic (PSL) encode model structure with weighted first-order logical clauses.
Learning these clauses from data is referred to as \emph{structure learning}. 
Structure learning alleviates the manual cost of specifying models.
However, this benefit comes with high computational costs; structure learning typically requires an expensive search over the space of clauses which involves repeated optimization of clause weights.
In this paper, we propose the first two approaches to structure learning for PSL.
We introduce a greedy search-based algorithm and a novel optimization method that trade-off scalability and approximations to the structure learning problem in varying ways. 
The highly scalable optimization method combines data-driven generation of clauses with a piecewise pseudolikelihood (PPLL) objective that learns model structure by optimizing clause weights only once.
We compare both methods across five real-world tasks, showing that PPLL achieves an order of magnitude runtime speedup and AUC gains up to 15\% over greedy search.
	
\end{abstract}
\section{Introduction}
\label{sec:intro}
Statistical relational learning (SRL) methods combine probabilistic reasoning with knowledge representations that capture the structure in problem domains.
Markov logic networks (MLN) \cite{richardson2006markov} and probabilistic soft logic (PSL) \cite{bach2017hinge} are notable SRL frameworks that define model structure with weighted first-order logic.
However, specifying logical clauses for each problem is laborious and requires domain knowledge.
The task of discovering these weighted clauses from data is referred to as \emph{structure learning}, and has been well-studied for MLNs \cite{kok2005learning,kok2009learning,kok2010learning,mihalkova2007bottom,biba2008discriminative,Huynh:2008:DSP:1390156.1390209,khosravi2010structure,khot2015gradient}.
The extensive related work for MLNs underscores the importance of structure learning for SRL.

Structure learning approaches alleviate the cost of model discovery. However, they face several critical computational challenges.
First, even when the model space is restricted to be finite, it results in a combinatorial search.  
Second, heuristic approaches that iteratively refine and grow a set of rules require interleaving of several costly rounds of parameter estimation and scoring.
Finally, scoring the model often involves computing the model likelihood which is typically intractable to evaluate exactly.

Structure learning approaches for MLNs vary in the degree to which they address these scalability challenges.
An efficient and extensible class of MLN structure learning algorithms adopt a \emph{bottom-up} strategy, mining patterns and motifs from training data to generate informative clauses \cite{mihalkova2007bottom,kok2009learning,kok2010learning}.
The data-driven heuristics reduce the search space to useful clauses but still interleave rounds of parameter estimation and scoring, which is expensive for SRL methods.

Motivated by the success of structure learning for MLNs, in this paper, we formalize the structure learning problem for PSL.
We extend the data-driven approach to generating clauses and propose two contrasting PSL structure learning methods that differ in scalability and choice of approximations.
We build on \emph{path-constrained} relational random walk methods \cite{lao2010relational,gardner2013improving} to generate clauses that capture patterns in the data.
To find the best set of clauses, we introduce a greedy search-based algorithm and an optimization method that uses a piecewise pseudolikelihood (PPLL) objective function.
PPLL decomposes the search over clauses into a single optimization over clause weights that is solved with an efficient parallel algorithm.
Our proposed PPLL approach addresses the scalability challenges of structure learning and its formulation can be easily extended to other SRL techniques, including MLNs.
In this paper, our key technical contributions are to:
\begin{itemize}
	\item[--] formulate path-constrained clause generation that efficiently finds relational patterns in the data.
	\item [--] propose greedy search and PPLL methods that select the best path-constrained clauses by trading off scalability and approximations for structure learning.
	\item[--] validate the predictive performance and runtimes of both methods with real-world tasks in biological paper recommendation, drug interaction prediction and knowledge base completion.
\end{itemize}
We compare both proposed PSL structure learning methods and show that our novel PPLL method achieves an order of magnitude runtime speedup and AUC improvements of up to 15\% over the greedy search method.
\section{Background}
\label{sec:background}

We briefly review of structure learning for statistical relational learning (SRL) and probabilistic soft logic (PSL), the framework for which we propose structure learning approaches.

\subsection{Structure Learning for SRL}
\label{subsec:srl}

Our work focuses on SRL methods such as MLNs and PSL that encode dependencies with first-order logic.
Below, we formalize the joint distributions defined using logical clauses before outlining structure learning for these methods.

An \textbf{atom} $p(\cdot)$ consists of a predicate $p$ (e.g. \textsc{Works, Lives}) over constants (e.g. \texttt{Alice, Bob}) or variables (e.g. $A, B$). An atom whose predicate arguments are all constants is a \emph{ground} atom.
A \textbf{literal} is an atom or its negation.
A \textbf{clause} $c$ is a formula $\land_{i}L_i \lor_j L_j$ where $L_i$ and $L_j$ are literals. 
Given $n$ clauses $C = \{c_1 \ldots c_n \}$ and real-valued weights $\mathbf{w} = \{w_1 \ldots w_n\}$, a \textbf{model} $\model = \{(w_1, c_1) \ldots (w_n, c_n) \}$ is a set of clause and weight pairs.

Given constants from a domain, we substitute the variables appearing in literals over $C$ with these constants to obtain a set of \emph{ground} clauses $G_c$ for each clause $c \in C$. The corresponding set of ground atoms is $\mathbf{X} = \{X_1 \ldots X_n\}$ where each $X_i$ is a random variable with assignments $\in \{0,1\}$.
The model $\model$ defines a distribution over $\mathbf{X}$ as:
\begin{equation}
\label{eq:srlpdf}
\begin{split}
P_{\model}(\mathbf{X}) &= \frac{1}{Z} \exp(-\sum_{i=1}^n \sum_{G_{c_i}} w_i\phi_{c_i}(\mathbf{X})) \\
&\text{where} \\
Z &= \sum_{\mathbf{X}} \exp(-\sum_{i=1}^n \sum_{G_{c_i}} w_i\phi_{c_i}(\mathbf{X}))
\end{split}
\end{equation}
Each $\phi_c$ instantiated from a clause $c$ is a function over assignments to $\mathbf{X}$ that returns 0 if $c$ is satisfied by $\mathbf{X}$ values and 1 otherwise.
Intuitively, assignments that satisfy more ground rules are exponentially more probable.

The problem of \emph{structure learning} finds the model $\model$ which best fits a set of observed assignments $\mathbf{X}$, regularized by model complexity. 
We denote the set of possible clauses as the language $\mathcal{L}$.
Although $\mathcal{L}$ can be infinite, it is standard to impose restrictions that make $\mathcal{L}$ finite for structure learning.
Formally, the structure learning problem finds $C \subseteq \mathcal{L}, \mathbf{w} \in \real, N = |C|$ that maximize a regularized log likelihood function $l_{ll}(C,\mathbf{w})$ given observed assignments:
\begin{equation}
\label{eq:genericslsplitobj}
\begin{split}
&{\arg \max}_{\mathbf{w} \in \real,\ C \subseteq \mathcal{L}} l_{ll}(C,\mathbf{w}) \\
&={\arg \max}_{\mathbf{w} \in \real,\ C \subseteq \mathcal{L}} \log P_{C,\mathbf{w}}(\mathbf{X}) - r(C,\mathbf{w})
\end{split}
\end{equation}
where $r(C,\mathbf{w})$ represents priors on the weights and structure. Typical choices for $r$ combine a Gaussian prior on weights and an exponential prior on clause length. 

The log likelihood requires an exponential sum to compute $Z$ and the optimization combines a combinatorial search over $\mathcal{L}$ with a maximization of continuous weights $\mathbf{w}$ (called weight learning).
Consequently, solving structure learning requires further approximations to search and scoring. 
Approaches to structure learning broadly interleave two key components: clause generation and model evaluation, or scoring. 
The clause generation phase produces a candidate language $\mathcal{L}$ over which to search.
In practice, $\mathcal{L}$ is a subset of all possible clauses, chosen to restrict the search to useful regions of the space.
Model evaluation typically iteratively refines the existing model by learning $\mathbf{w}$ and scoring candidate clauses in $\mathcal{L}$ using approximations to $l_{ll}(C,\textbf{w})$.
\subsection{Probabilistic Soft Logic}
\label{subsec:hlmrf}

Probabilistic soft logic (PSL) is a SRL framework that defines hinge-loss Markov random fields, a special class of the undirected graphical model given by \eqnref{eq:srlpdf}.
HL-MRFs are conditional distributions over real-valued atom assignments in $[0, 1]$ and apply a continuous relaxation of Boolean logic to the ground clauses to derive $\phi_c$ of the form:
\begin{equation}
\label{eq:energy}
\phi_c(\mathbf{X}) = \max \{1 - \sum_{i \in I^+} X_i - \sum_{i \in I^-} ( 1 - X_i), 0\}^p 
\end{equation}
where $I^+$ and $I^-$ denote the set of non-negated and negated ground atoms in the clause and $p \in \{1, 2\}$.
In contrast to ground Boolean clauses that are satisfied or violated, a ground clause in soft logic returns a continuous distance to satisfaction. 
Intuitively, $\phi_c(\mathbf{X})$ corresponds to a linear or quadratic penalty for violating clause $c$.

PSL defines distributions over the target variables for a particular task conditioned on the remaining evidence variables.
Formally, given a set of \emph{target predicates} $\mathbb{P}_T$, a PSL model $\pslmodel$ consists of non-negative weights $\mathbf{w} \in \posreal$ and disjunctive clauses $\pslclause$ where the predicate for literal $T_i$ belongs to $\mathbb{P}_T$.
Given a set of atoms $\mathbf{Y}$ where random variable $Y_i \in [0, 1]$ and a set of evidence atoms $\mathbf{X}$ where each $X_i \in [0, 1]$ is an observed variable, a PSL model $\pslmodel$ defines an HL-MRF distribution of the form:
\begin{equation}
\label{eq:hlmrf}
\begin{split}
P_{\pslmodel}(\mathbf{Y} |\mathbf{X}) &= \frac{1}{Z} \exp(-\sum_{i=1}^n \sum_{G_{c_i}} w_i\phi_{c_i}(\mathbf{X},\mathbf{Y})) \\
&\text{where} \\
Z &= \int_\mathbf{Y} \exp(-\sum_{i=1}^n \sum_{G_{c_i}} w_i\phi_{c_i}(\mathbf{X},\mathbf{Y})) 
\end{split}
\end{equation}

PSL has been successfully applied to many problem including natural language processing \cite{beltagy2014probabilistic}, social media analysis \cite{johnson2016all,ebrahimi2016weakly} and information extraction \cite{platanios2017estimating}.


\section{Structure Learning for PSL}
\label{sec:problem}

Given target predicates $\mathbb{P}_T$, structure learning for PSL finds a model $\pslmodel$ to infer $t_i \in \mathbb{P}_T$.
We denote language space for PSL $\mathcal{L}_R$, which is restricted to clauses of the form $\pslclause$. We again constrain $\mathcal{L}_R$ to be finite.
To overcome the intractable likelihood score, pseudo-likelihood \cite{besag1975statistical} is an approximation that is commonly used across SRL structure learning and weight learning methods.
For HL-MRFs, the pseudo-likelihood $\hat{P}_{\pslmodel}$ approximates the likelihood as:
\begin{equation}
\label{eq:pseudolikelihood}
\begin{split}
&\hat{P}_{\pslmodel}(\mathbf{Y}|\mathbf{X}) = \prod_{Y_i \in \mathbf{Y}} \frac{1}{Z_i(\mathbf{Y},\mathbf{X})} \exp( -f_i(Y_i, \mathbf{Y},\mathbf{X}))  \\
&\text{where} \\
&Z_i(\mathbf{Y},\mathbf{X}) = \int_{Y_i} \exp(-f_i(Y_i, \mathbf{Y},\mathbf{X})) \\
&f_i(Y_i, \mathbf{Y},\mathbf{X}) = \sum_{c \in C} \sum_{j:Y_i \in G_c} w_j\phi_j(Y_i,\mathbf{X},\mathbf{Y}) 
\end{split}
\end{equation}
The notation $j:Y_i \in G_c$ selects ground clauses $j$ where $Y_i$ appears.

Given target predicates $\mathbb{P}_T$, real-valued variable assignments $\mathbf{Y}$ and $\mathbf{X}$ where each $Y_i$ atom consists of $p \in \mathbb{P}_T$, following the objective in Equation \ref{eq:genericslsplitobj}, structure learning for PSL maximizes log pseudolikelihood $l_{pll}(C, \posweight)$:
\begin{equation}
\label{eq:psl-slobj}
{\arg \max}_{C \subseteq \mathcal{L}_R, \posweight \in \posreal} \sum_{Y_i \in Y} -\log(Z_i) - \mathbf{w}^T\Phi_C(\mathbf{X},\mathbf{Y})  \\
\end{equation}
where $\Phi_C$ denotes all ground rules that can be instantiated from clauses $C$.
In the next section, we propose two approaches to the structure learning problem for HL-MRFs that rely on an efficient clause generation algorithm.
\section{Approaches to PSL Structure Learning}
\label{sec:approach}

To formulate PSL structure learning algorithms, we introduce approaches for both key method components: clause generation and model evaluation.
We outline an efficient algorithm for data-driven clause generation. 
For model evaluation over these clauses, we first propose a straightforward greedy local search algorithm (GLS). 
To improve upon the computationally expensive search-based approach, we introduce a novel optimization approach, piecewise pseudo-likelihood (PPLL). 
PPLL unifies the efficient clause generation with a surrogate convex objective that can be optimized exactly and in parallel. 

\subsection{Path-Constrained Clause Generation}
\label{subsec:clause}

The clause generation phase of structure learning outputs the language $\mathcal{L}_R$ of first-order logic clauses over which to search.
Driven by relational random walk methods used for information retrieval tasks \cite{lao2011random,gardner2013improving}, we formulate a special class of \emph{path-constrained clauses} that capture relational patterns in the data.
Path-constrained clause generation is also related to the pre-processing steps in bottom-up structure learning methods \cite{mihalkova2007bottom,kok2009learning,kok2010learning}.
Bottom-up methods typically use relational paths as heuristics to cluster predicates into templates and enumerate all clauses that contain predicate literals from the same template.
The structure learning algorithm greedily selects from these clauses.
Path-constrained clause generation also produces $\mathcal{L}_R$ prior to structure learning.
Here, we use a breadth-first traversal algorithm which directly generates informative path-constrained clauses by variablizing relational paths in the data.

The inputs to path-constrained clause generation are the ground atoms of a domain, the set of all predicates $\mathbb{P}$ and target predicate $\mathbb{P}_T$. In this work, we consider predicates with arity of two but our approach will be extended to support predicates with arity three and higher.
We begin with a running example that illustrates the definitions below.
\begin{example}
	\label{eg:reldata}
	Consider a ground atom set with \textsc{Cites}(Paper1, Paper2), \textsc{Mentions}(Paper2, Gene), \textsc{Mentions}(Paper1, Gene) and $\mathbb{P}_T = \{\textsc{Mentions}\}$. 
	In this simple example, all ground atoms have an assignment of 1. 
	In general, real-valued assignments to atoms must be rounded to 0 or 1 during path-constrained clause generation.
\end{example}

\begin{definition}
A \textbf{target relational path} for $t_i \in \mathbb{P}_T$ denoted $\pi_j^{t_i}$ is defined by an ordered list of ground atoms $[p_1(e_1, e_2), p_2(e_2, e_3) \ldots, p_s(e_{s}, e_{s+1}), t_i(e_1, e_{s+1})]$ such that
each $p_i(e_i, e_{i+1}) = 1$, its last argument $e_{i+1}$ is the first argument of $p_{i+1}(e_{i+1})$, and $t_i(e_1, e_{s+1}) \in \{0, 1\}$ is a target atom.
\end{definition}
\begin{definition}
Given a target relational path $\pi_j^{t_i}$, the corresponding \emph{first-order} \textbf{path-constrained clause} $c_{\pi_j}^{t_i}$ has the form $p_1(E_1, E_2) \wedge \ldots \wedge p_s(E_s, E_{s+1}) \rightarrow t_i(E_1, E_{s+1})$ where each $E_i$ is a logical variable and the $j$-th literal in the clause variablizes the $j$-th atom in $\pi_j^{t_i}$.
The \textbf{negation} of $c_{\pi_j}^{t_i}$ is the clause with $\neg t_i(E_1, E_{s+1})$, the target predicate literal negated.
\end{definition}

For \egref{eg:reldata}, given target relational path [\textsc{Cites}(Paper1, Paper2), \textsc{Mentions}(Paper2, Gene), \textsc{Mentions}(Paper1, Gene)], we obtain the first-order path-constrained clause:
\begin{equation*}
\footnotesize
\begin{split}
\textsc{Cites}(E_1, E_2) &\wedge \textsc{Mentions}(E_2, E_3) \rightarrow \textsc{Mentions}(E_1, E_3)
\end{split}
\end{equation*}

We generate the set of all possible path-constrained clauses $C_\Pi$ up to length $s$, by performing breadth-first search (BFS) of up to depth $s$ from the first argument $e_j$ of each target atom $t_i(e_j, e_k)$. 
\begin{definition}
	A \textbf{connected BFS search tree} $b^i_{jk}$ for training example $t_i(e_j, e_k)$ is rooted at $e_j$ and one of its leaf nodes \emph{must} be $e_k$.
	Every non-leaf constant $e_u$ in $b^i_{jk}$ has child entities $e_v$ connected by ground atoms $p_i(e_u, e_v) = 1$.
\end{definition}
For \egref{eg:reldata}, the connected BFS search tree of depth $2$ for target atom \textsc{Mentions}(Paper1, Gene) is:
$$\text{Paper1} \xrightarrow{\textsc{Cites}} \text{Paper2}  \xrightarrow{\textsc{Mentions}} \text{Gene} $$

Given a tree $b^i_{jk}$, each path from its root $e_j$ to leaf node $e_k$ is a target relational path $\pi_j^{t_i}$.
For target predicate $t_i$, $B^i = \{b_1 \ldots b_n\}$ is the set of connected BFS search trees corresponding to all $n$ target atoms.
For all $t_i \in \mathbb{P}_T$, we enumerate all such $\pi_i^{t_i}$ from each $b \in B^i$ and obtain the unique set of these paths $\Pi$. 
For each $\pi_i \in \Pi$, we form the corresponding path-constrained clause and its negation to obtain all such clauses $C_\Pi$.
Moreover, we can further restrict $C_\Pi$ to those clauses that connect $\geq t$ target atoms, preferring clauses that cover, or explain, at least training $t$ examples.
The language defined by $C_\Pi$ guides the search over models that capture informative relational patterns in the data.
Although $C_\Pi$ produces only Horn clauses and is thus a subset of the language $\mathcal{L}_R$ \cite{kazemi2018bridging}, it has been successfully used in several relational learning tasks \cite{lao2010relational,gardner2013improving}.
While our path-constrained clause generation performs well in the tasks we study, where needed, we will explore more expressive strategies.
 
\subsection{Greedy Local Search}
\label{subsec:search}

Given $N$ path-constrained clauses, exactly maximizing the pseudolikelihood objective given by \eqnref{eq:pseudolikelihood} requires evaluating $2^N$ subsets of clauses, which is already infeasible with only 100 clauses. Instead, we propose an approximate greedy search algorithm that selects locally optimal clauses in each iteration to maximize pseudolikelihood.

\begin{algorithm}
	\caption{\textbf{Greedy Local Search} (GLS)}
	\label{alg:local}
	\algrenewcommand\algorithmicrequire{\textbf{Input:}}
	\algrenewcommand\algorithmicensure{\textbf{Output:}}
	\begin{algorithmic}
		\Require $C_\Pi$: path-constrained clauses; $\epsilon$: tolerance; $l$:  max iterations
		\Ensure $C^*, \mathbf{w}$: optimal clauses and weights
		\State $S \gets C_\Pi$
		\State $C^* \gets \emptyset$
		\State $current, prev, i \gets 0$
		\While{$current - prev \geq \epsilon$ or $i \leq l$}
			\State $current \gets prev$
			\For{$s \in S$}
				\State $C^* \gets C^* \cup s$
				\State $score \gets \max_{\mathbf{w}} l_{pll}(C^*, \mathbf{w})$
				\If{$score > current$}
					\State $current \gets score$
					\State $c^* \gets s$
				\EndIf
				\State $C^* \gets C^* \setminus s$
			\EndFor
			\State $C^* \gets C^* \cup c^*$
			\State $S \gets S \setminus c^*$
			\State $i \gets i + 1$
		\EndWhile
	\end{algorithmic}
\end{algorithm}

\algoref{alg:local} gives the pseudocode for \textbf{greedy local search} (GLS) which approximately maximizes the pseudolikelihood score $l_{pll}(\cdot)$. 
GLS iteratively picks the $c^* \in C_\Pi$ that maximizes $l_{pll}(\cdot)$ and adds it to the model $M$ until the score has only improved by $\le \epsilon$ or a maximum number of iterations $l$ has been reached.
While GLS is straightforward to implement, it requires $O(Nl)$ rounds of weight learning and evaluating $l_{pll}(\cdot)$ where $N$ denotes the size of $C_\Pi$. As $N$ grows, the GLS becomes prohibitively expensive unless we sacrifice performance by increasing $\epsilon$ or decreasing $l$. To overcome the scalability pitfalls of GLS and search-based methods at large, we introduce a new structure learning objective that can be optimized efficiently and exactly.
\subsection{Piecewise Pseudolikelihood}
\label{subsec:piecewise}

The partition function $Z_i$ in pseudo-likelihood involves an integration that couples all model clauses.
Optimizing pseudo-likelihood requires evaluating all subsets of the language $\mathcal{L}_R$, necessitating greedy approximations to the combinatorial problem.
To overcome this computational bottleneck, we propose a new, efficient-to-optimize objective function called \textbf{piecewise pseudolikelihood} (PPLL). 
Below, we derive two key results which have significant consequences for scalability of structure learning: 1) with PPLL, structure learning is solved by performing weight learning once; and 2) the factorization used by PPLL admits an inherently parallelizable gradient-based algorithm for optimization.

PPLL was first proposed for weight learning in conditional random fields (CRF) \cite{sutton2007piecewise}.
For HL-MRFs, PPLL factorizes the joint conditional distribution along both random variables and clauses and is defined as:
\begin{equation}
\label{eq:piecewiselikelihood}
\begin{split}
&P^*_{\pslmodel}(\mathbf{Y} | \mathbf{X}) = \prod_{c \in C} \prod_{Y_i \in \mathbf{Y}}  \frac{\exp( -f^c_i(Y_i, \mathbf{Y}, \mathbf{X}))}{Z_i^c(\mathbf{Y},\mathbf{X})}   \\
&\text{where} \\
&Z_i^c(\mathbf{Y}, \mathbf{X}) = \int_{Y_i} \exp(-f_i^c(Y_i, \mathbf{Y}, \mathbf{X})) \\
&f_i^c(Y_i, \mathbf{Y}, \mathbf{X}) = \sum_{j:Y_i \in G_c} w_j\phi_j(Y_i, \mathbf{Y}, \mathbf{X})
\end{split}
\end{equation}
The key advantage of PPLL over pseudo-likelihood arises from the factorization of $Z_i$ into $Z_i^c$, which requires only clause $c$ and variable $Y_i$ for its computation.

Following standard convention for structure learning, we optimize the log of PPLL denoted $l_{ppll}(C,\mathbf{w})$.
We highlight a connection between PPLL and pseudolikelihood that is useful in deriving the two key scalability results of PPLL. 
The product of terms in PPLL corresponding to clause $c$ is the log pseudo-likelihood of the model containing only clause $c$. We denote this $l_{pll}^c(w_c)$:
\begin{equation}
l_{pll}^c(w_c) = \sum_{Y_i \in \mathbf{Y}} -\log(Z_i^c(\mathbf{Y}, \mathbf{X})) - f^c_i(Y_i, \mathbf{Y}, \mathbf{X})
\end{equation}

We now show that for the log PPLL objective function, performing weight learning on the model containing all clauses in $\mathcal{L}_R$ is equivalent to optimizing the objective function over the space of all models. Formally:
\begin{equation}
\begin{split}
&{\arg \max}_{\ C \subseteq \mathcal{L}_R, \textbf{w} \in \posreal} l_{ppll}(C,\textbf{w}) \\
&\equiv \\
&{\arg \max}_{\textbf{w} \in \posreal} l_{ppll}(\mathcal{L}_R,\textbf{w})
\end{split}
\end{equation}


\begin{lemma}
\label{lemma:decompose}
Optimizing $l_{ppll}(C,\textbf{w})$ over the set of weights $\textbf{w}$ is equivalent to optimizing over each $w_c$ separately.
\end{lemma}
\begin{proof}
Each $l^c_{pll}(w_c)$ is a function of only $w_c$. By definition of $l_{ppll}(C,\textbf{w})$, we have
\begin{equation*}
\begin{split}
\arg \max_{\textbf{w} \in \posreal} & l_{ppll}(C, \textbf{w}) = \arg \max_{\textbf{w} \in \posreal} \sum_{c \in C} l^c_{pll}(w_c) \\
				&= \sum_{c \in C} \arg \max_{w_c \in \posreal} l^r_{pll}(w_c)
\end{split}
\end{equation*} 
\end{proof}
\begin{theorem}
\label{thm:ppll}
For PPLL, maximizing the weights $\mathbf{w}$ of the model containing all clauses in $\mathcal{L}_R$ is equivalent to optimizing the structure learning objective. 
\end{theorem}
\begin{proof}
\begin{equation*}
\begin{split}
& {\arg \max}_{\ C \subseteq \mathcal{L}_R, \textbf{w} \in \posreal} l_{ppll}(C,\textbf{w}) \\
& = {\arg \max}_{\ C \subseteq \mathcal{L}_R} \sum_{c \in C} \arg \max_{w_c \in \posreal} l^c_{pll}(w_c)\ [\text{Lemma } \ref{lemma:decompose}]\\
&\text{By setting $w_c = 0$, we get $l_{pll}^r(w_c) = 0$.} \\
&\text{Therefore, the maxima must be non-negative, i.e.:} \\
	&\arg \max_{w_c \in \posreal}  l^c_{pll}(w_c) \geq 0. \\
& \text{This implies that:} \\ 
&{\arg \max}_{\ C \subseteq \mathcal{L}_R} \sum_{c \in C} \arg \max_{w_c \in \posreal} l^c_{pll}(w_c) \\
& = \sum_{c \in \mathcal{L}_R} \arg \max_{w_c \in \posreal} l^c_{pll}(w_c) \\
& = {\arg \max}_{\textbf{w} \in \posreal} l_{ppll}(\mathcal{L}_R,\textbf{w})
\end{split}
\end{equation*}
\end{proof}

As a result of \thmref{thm:ppll}, instead of combinatorial search, we perform a simpler continuous optimization over weights that can be solved efficiently.
Since the objective is convex, and the weights are non-negative, we optimize the above objective using projected gradient descent. 

The projected gradient descent algorithm for optimizing the objective function is shown in \algoref{alg:piecewise}.
The partial derivative of $l_{ppll}(C, \textbf{w})$ for a given weight $w_c$ is of the form:
\begin{equation}
\label{eq:ppllgrad}
\begin{split}
&\nabla_{w_c} = \Phi_c(Y_i, \mathbf{Y}, \mathbf{X}) - \mathbb{E}_{ppll}[\Phi_c(Y_i, \mathbf{Y}, \mathbf{X})] \\ 
&\text{where} \\ 
&\Phi_c(Y_i, \mathbf{Y}, \mathbf{X}) = \sum_{Y_i \in \mathbf{Y}} \sum_{j: Y_i \in G_c} \phi_c(Y_i, \mathbf{Y}, \mathbf{X})
\end{split}
\end{equation}
The gradient for any weight $w_c$ is the difference between observed and expected penalties summed over corresponding ground clauses $G_c$. 
For both pseudo-likelihood and PPLL, we can compute observed penalties once and cache their values but the repeated expected value computations, even for a one-dimensional integral, remain costly.
However, unlike the gradients for pseudo-likelihood, each expectation term in the PPLL gradient considers a single clause. 
Thus, when evaluating gradients for weight updates in \algoref{alg:piecewise}, we use multi-threading to compute the expectation terms in parallel. 
The dual advantages of parallelizing and requiring weight learning only once makes PPLL highly scalable.
After convergence of the gradient descent procedure, we return the set of clauses with non-zero weights as the final model.

\begin{algorithm}
	\caption{\textbf{Piecewise Pseudolikelihood} (PPLL)}
	\label{alg:piecewise}
	\algrenewcommand\algorithmicrequire{\textbf{Input:}}
	\algrenewcommand\algorithmicensure{\textbf{Output:}}
	\begin{algorithmic}
		\Require $C_\Pi$: path-constrained clauses; $\epsilon$: tolerance; $l$: max iterations; $\alpha$: step size
		\Ensure $C^*, \mathbf{w}$: optimal clauses and weights
		\For {$c \in C_\Pi$}
			\State $C^* \gets c$ 
		\EndFor
		\State $i \gets 0$
		\State $score_{prev} \gets -\infty$
		\State $score_{curr} \gets l_{ppll}$
		\While{$score_{curr} - score_{prev} > \epsilon$ or $i < l$}
			\State $i \gets i+1$
			\For {$c \in C^*$}
				\State $w_c \gets w_c + \alpha \nabla_{w_c}$
				\If{ $w_c < 0$}
					\State $w_c = 0$
				\EndIf
			\EndFor
			\State $score_{prev} \gets score_{curr}$
			\State $score_{curr} \gets l_{ppll}$
		\EndWhile
		\For {$c \in C^*$}
			\If{ $w_c = 0$ }
				\State $C^* \gets C^* \setminus c$ 
			\EndIf
		\EndFor
	\end{algorithmic}
\end{algorithm}


\section{Experimental Evaluation}
\label{sec:eval}
 
\begin{table*}[th!]
	\small
	\centering
	\caption{Average AUC of methods across five prediction tasks. Bolded numbers are statistically significant at $\alpha=0.05$. We show that PPLL training improves over GLS in three out of five settings.}	
	\label{tab:auc}
	\begin{tabular}{l l l l l l}
		\toprule
		\textbf{Method} &  Fly-\textsc{Gene} & Yeast-\textsc{Gene} & DDI-\textsc{Interacts} & Freebase-\textsc{FilmRating} & Freebase-\textsc{BookAuthor}\\
		\midrule
		GLS & 0.95 $\pm$ 0.01 & 0.86 $\pm$ 0.02 & 0.66 $\pm$ 0.01 & 0.65 $\pm$ 0.04 & 0.67 $\pm$ 0.03 \\
		PPLL & {\bf 0.97 $\pm$ 0.002} & {\bf 0.90 $\pm$ 0.003}  & {\bf 0.76 $\pm$ 0.01} & 0.65 $\pm$ 0.05 & 0.65 $\pm$ 0.04 \\
		\bottomrule
	\end{tabular}
\end{table*}

\begin{figure*}
		\centering
		\label{fig:runtime}
		\includegraphics[width=8.5cm]{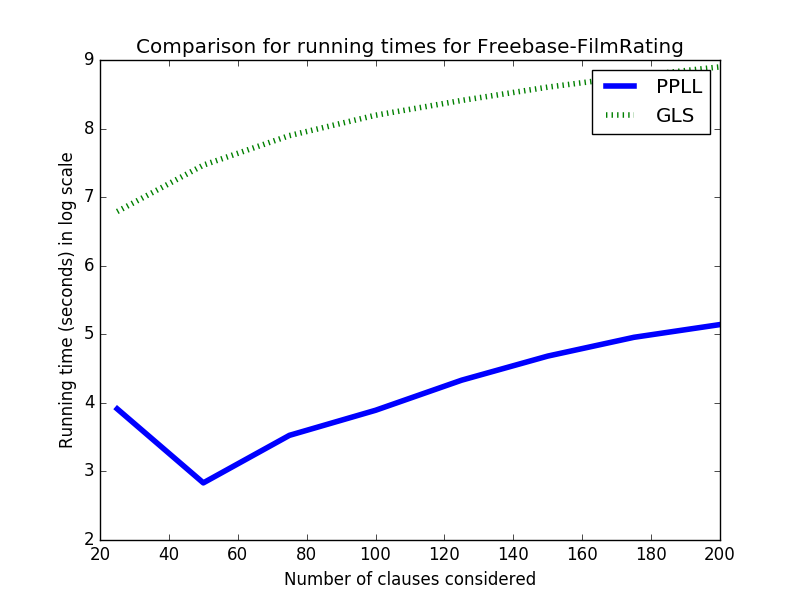}	
		\includegraphics[width=8.5cm]{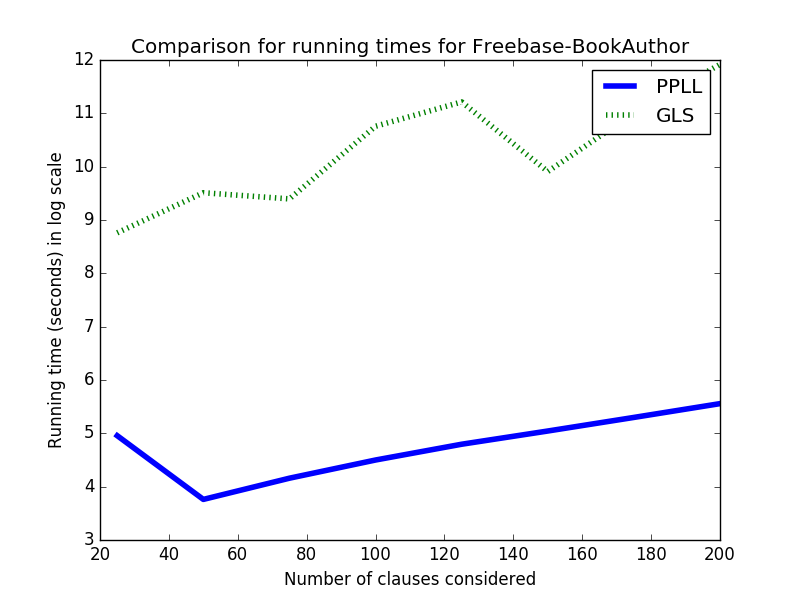}
		\caption{Running times (in seconds) in log scale on Freebase tasks. PPLL consistently scales more effectively than GLS.}	
		
\end{figure*}

%
 
The PPLL optimization method uses a fully factorized approximation for scalability while GLS greedily maximizes the less decoupled pseudolikelihood at the expense of speed.
We explore the trade-offs made by these two methods by evaluating predictive performance and scalability.
We investigate these experimental questions with five prediction tasks and compare PPLL against GLS after generating path-constrained clauses.
The evaluation tasks include paper recommendation in biological citation networks, drug interaction prediction and knowledge base completion.

\subsection{Datasets}
For our datasets, we obtain citation networks for biological publications, drug-drug interaction pharmacological networks and knowledge graphs.

\paragraph{Biological Citation Networks}

Our first dataset consists of biology-related papers and entities such as authors, venues, words, genes, proteins and chemical compounds \cite{lao2011random}. The dataset includes relations over these entity types for two domains, ``Fly'' and ``Yeast'', resulting in two citation networks. The prediction target is the \textsc{Gene} relation between genes and papers that mention them. To enforce training only on papers from the past, we partition papers into periods of time, using those from 2006 as observations, training on papers from 2007 and evaluating on papers from 2008. We randomly subsample targets to obtain 1500 train and test links, and generate five such random splits for cross-validation.

\paragraph{Drug-drug interaction}

The second dataset we use includes chemical interactions between drug pairs, called drug-drug interactions (DDI) across 196 drug compounds obtained from the DrugBank database. This dataset also contains a directed graph of relations from Drugbank between these drugs and gene targets, enzymes, and transporters. Our target for prediction is the \textsc{Interacts} relation between drugs. We subsample the tens of thousands of labeled interaction and shuffle the remaining labeled DDI links into five folds for cross-validation. Each fold contains almost 2000 labeled DDI targets. We alternate using one fold of DDI edges as observations, one for training and one for held-out evaluation.

\paragraph{Freebase}
Our third dataset comes from the Freebase knowledge graph and is well-used in validating knowledge base (KB) completion tasks \cite{gardner2014incorporating}. We study KB completion for two relations: links between films and their ratings (\textsc{FilmRating}$(\cdot)$) and links between authors and books written (\textsc{BookAuthor}$(\cdot)$). The remaining relations in the KB are observed. For both target relations, we subsample edges and split the resultant edges into five folds for cross-validation, yielding 1000 labeled edges per fold.

\subsection{Experimental Setup}

Our first experimental question evaluates predictive performance using area under the ROC curve (AUC) on held-out data with five-fold cross-validation across the five tasks described above. 
Our second question validates scalability by comparing running-times for both methods as the number of clauses grows.
For both methods, we use ADMM inference implemented in the probabilistic soft logic (PSL) framework \cite{bach2017hinge}.
For GLS, we use the pseudo-likelihood learning algorithm in PSL and implement its corresponding scoring function within in PSL \footnote{\texttt{psl.linqs.org}}.
For PPLL, we implement the parallelized learning algorithm in PSL. 
For all tasks, we enumerate target relational paths using the BFS utility in the Path Ranking Algorithm (PRA) \footnote{\texttt{github.com/matt-gardner/pra}} \cite{lao2010relational,gardner2013improving,gardner2014incorporating}
and generate path-constrained clauses from these paths.
PRA generates and includes the inverses of all atoms when performing BFS. To form clause literals from these inverses, we use the original predicate and reverse the order of its variablized arguments.

As the number of generated clauses grows, GLS becomes prohibitive as we show in our scalability results and necessitates a clause-pruning strategy.
We prune the set of clauses by retaining those that connect at least 10 target atoms and select the top 50 clauses by number of targets connected. 
For each target predicate $t_i$ in the prediction tasks detailed above, we also add a negative prior clause $\neg t_i(\cdot)$ to the candidate clauses.
For link prediction tasks, the negative prior captures the intuition that true positive links are rare and most links do not form.
We refer the reader to \cite{bach2017hinge} for detailed discussion on the importance of negative priors.
For the biological citation networks and Freebase settings, we subsample negative examples of the targets to mitigate the imbalance in labeled training data.
We perform 150 iterations of gradient descent for PPLL and 15 for GLS since it requires several rounds of weight learning.

\subsection{Predictive Performance}

Our first experimental question investigates the ramifications for predictive performance of the approximations made by each method.
PPLL approximates the likelihood by fully factorizing across clauses and target variables while GLS uses the pseudolikelihood approximation which still couples clauses.
We examine whether the decoupling in PPLL limits its predictive performance.
We generate path-constrained clauses as input to both methods and evaluate their performance on held-out data.
Table $\ref{tab:auc}$ compares both methods using AUC for all five prediction tasks averaged across multiple folds and splits. 

Table $\ref{tab:auc}$ shows that PPLL gains significantly in AUC over GLS in three out of five settings.
For the \textsc{Gene} link prediction task in the Yeast and Fly biological citation networks, PPLL also yields lower variance given the same rules. 
In the DDI setting where we predict \textsc{Interacts} links between drugs, PPLL enjoys a 15\% AUC gain over GLS from 0.66 to 0.76.
In the Freebase setting, for the \textsc{BookAuthor} task, PPLL again achieves comparable performance with GLS. 
GLS only improves slightly over the PPLL approximation in one setting, predicting \textsc{FilmRating} with a statistically insignificant gain of 0.02 in AUC.

\subsection{Scalability Study}

Our second experimental question focuses on the scalability trade-offs made by GLS and PPLL. 
PPLL requires weight learning over clauses, made faster with parallelized updates while GLS requires iterative rounds of weight learning and model evaluation.
We select the two Freebase tasks, \textsc{BookAuthor} and \textsc{FilmRating} where path-constrained clause generation initially yielded several hundred rules.
We plot the running time for both methods as the size of the candidate clause set increases from 25 to 200. 

Figure \ref{fig:runtime} shows the running times (in seconds) for both methods plotted in log scale across the two Freebase tasks as the number of clauses to evaluate increases. 
The results show that while PPLL remains computationally feasible as the number of clauses increases, GLS quickly becomes intractable as the clause set grows.
Indeed, for \textsc{BookAuthor}, GLS requires almost two days to learn a model with 200 candidate clauses. In contrast, PPLL completes in four minutes using 200 clauses in the same setting.
PPLL overcomes the requirement of interleaving weight learning and scoring while also admitting parallel weight learning updates, boosting scalability.
The results suggest that PPLL can explore a larger space of models in significantly less time.
\section{Related Work}
\label{sec:related}

	Finally, we review related work on structure learning approaches for undirected graphical models, which underpin the SRL methods we highlight in this paper. We also provide an overview of work in relational information retrieval which motivates our path-constrained clause generation. 

	For general Markov random fields (MRF) and their conditional variants, structure learning typically induces feature functions represented as propositional logical clauses of boolean attributes \cite{mccallum2002efficiently,davis2010bottom}. 
	An approximate model score is optimized with a greedy search that iteratively picks clausal feature functions to include while refining candidate features by adding, removing or negating literals to single-literal clauses.
	MRF structure learning is also viewed as a feature selection problem solved by performing L1-regularized optimization over candidate features, admitting fast gradient descent and online algorithms \cite{perkins2003grafting,zhu2010grafting}.
	
	Although structure learning has not been studied in PSL, many algorithms have been proposed to learn MLNs.
	The initial approach to MLN structure learning performs greedy beam search to grow the set of model clauses starting from single-literal clauses. The clause generation performs all possible negations and additions to an existing set of clauses while the search procedure iteratively selects clauses to refine.
	To efficiently guide the search towards useful models, bottom-up approaches generate informative clauses by using relational paths to capture patterns and motifs in the data \cite{mihalkova2007bottom,kok2009learning,kok2010learning}.
	This relational path mining in bottom-up approaches is related to the path ranking algorithm (PRA) for relational information retrieval \cite{lao2010relational}.
	PRA performs random walks or breadth-first traversal on relational data to find useful path-based features for retrieval tasks \cite{lao2010relational,gardner2013improving,gardner2014incorporating}.
	
	Most recently, MLN structure learning has been viewed from the perspectives of moralizing learned Bayesian networks \cite{khosravi2010structure} and functional gradient boosting \cite{khot2011learning,khot2015gradient}. These methods improve scalability while maintaining predictive performance.
	Alternately, approaches have been proposed to learn MLNs for target variables specific to a task of interest as we do for PSL. Structure learning methods for particular tasks use inductive logic programming \cite{muggleton1991inductive} to generate clauses which are pruned with L1-regularized learning \cite{Huynh:2008:DSP:1390156.1390209,huynh2011online} or perform iterative local search \cite{biba2008discriminative} to refine rules with the operations described above.
	
	
%
\vspace*{-0.15cm}
\section{Conclusion and Future Work}
\label{sec:disc}
In this work, we formalize the structure learning problem for PSL and introduce an efficient-to-optimize and convex surrogate objective function, PPLL. 
We unify scalable optimization with data-driven path-constrained clause generation. 
Compared to the straightforward but inefficient greedy local search method, PPLL remains scalable as the space of candidate rules grows and demonstrates good predictive performance across five real-world tasks. 
Although we focus on PSL in this work, our PPLL method can be generalized for MLNs and other SRL frameworks.
An important line of future work for PSL structure learning is extending L1-regularized feature selection and functional gradient boosting approaches which have been applied successfully to MRFs and MLNs. These methods have been shown to scale while maintaining good predictive performance.
\vspace*{-0.15cm}
\section*{Acknowledgements}
\label{sec:ack}
This work is sponsored by the Air Force Research Laboratory (AFRL) and Defense Advanced Research Projects Agency (DARPA), and supported by NSF grants CCF-1740850 and NSF IIS-1703331. We thank Sriraam Natarajan and Devendra Singh Dhami for sharing their DrugBank dataset.
\bibliographystyle{aaai}
\bibliography{structurelearning,pslpapers}

\end{document}